\documentclass{article}

\usepackage{arxiv}
\usepackage{cite}
\usepackage{amsmath,amssymb,amsfonts}
\usepackage{color,soul}

\usepackage{algorithm}
\usepackage[noend]{algorithmic}
\usepackage{graphicx}
\usepackage{hyperref}
\usepackage{cleveref}
\usepackage{textcomp}
\usepackage{xcolor}
\usepackage{tikz}
\usepackage{amsthm}
\usepackage{dsfont}
\usepackage{flushend}
\newtheorem{theorem}{Theorem}
\newtheorem{corollary}{Corollary}
\newtheorem{definition}{Definition}
\newtheorem{proposition}{Proposition}
\usetikzlibrary{positioning}
\def\BibTeX{{\rm B\kern-.05em{\sc i\kern-.025em b}\kern-.08em
    T\kern-.1667em\lower.7ex\hbox{E}\kern-.125emX}}

\begin{document}
\newcommand{\SWITCH}[1]{\STATE \textbf{switch} (#1)}
\newcommand{\ENDSWITCH}{\STATE \textbf{end switch}}
\newcommand{\CASE}[1]{\STATE \textbf{case} #1\textbf{:} \begin{ALC@g}}
\newcommand{\ENDCASE}{\end{ALC@g}}
\newcommand{\CASELINE}[1]{\STATE \textbf{case} #1\textbf{:} }
\newcommand{\DEFAULT}{\STATE \textbf{default:} \begin{ALC@g}}
\newcommand{\ENDDEFAULT}{\end{ALC@g}}
\newcommand{\DEFAULTLINE}[1]{\STATE \textbf{default:} }
\renewcommand{\algorithmicrequire}{\textbf{Input:}}
\renewcommand{\algorithmicensure}{\textbf{Output:}}

\title{Certified Data Removal in Sum-Product Networks}

\author{Alexander Becker\\
\textit{Chair of Artificial Intelligence, Computer Science} \\
\textit{TU Dortmund University} \\
Dortmund, Germany \\
alexander2.becker@tu-dortmund.de
\And
Thomas Liebig \\
\textit{Chair of Artificial Intelligence, Computer Science} \\
\textit{TU Dortmund University} \\
Dortmund, Germany \\
thomas.liebig@tu-dortmund.de
}

\maketitle

\begin{abstract}
Data protection regulations like the GDPR or the California Consumer Privacy Act give users more control over the data that is collected about them. Deleting the collected data is often insufficient to guarantee data privacy since it is often used to train machine learning models, which can expose information about the training data. Thus, a guarantee that a trained model does not expose information about its training data is additionally needed. In this paper, we present \textsc{UnlearnSPN} -- an algorithm that removes the influence of single data points from a trained sum-product network and thereby allows fulfilling data privacy requirements on demand.
\end{abstract}


\thanks{This research has been funded by the Federal Ministry of Education and Research of Germany and the state of North-Rhine Westphalia as part of the  competence center for machine learning ML2R (01–S18038A) and the Lamarr-Institute for Machine Learning and Artificial Intelligence (LAMARR22B).}

\section{Introduction}
    Due to legal requirements like the European General Data Protection Regulation (GDPR), the California Consumer Privacy Act, and many others, users gain more control over their personal data collected daily. The \textit{right to be forgotten} is of particular importance, which states that collected data must be deleted when requested. Deleting data is often insufficient to provide real data privacy. This is especially the case if the data was used to train machine learning models since they might expose information about their training data via white-box or even black-box access. Motivated by this, the field of Machine Unlearning and Forgetting gained more and more attention. So far, research mostly focused on unlearning in deep neural networks \cite{bourtoule2021machine,golatkar2020eternal,golatkar2021mixed,graves2020amnesiac,guo2019certified} and linear models \cite{aldaghri2021coded,golatkar2020eternal,guo2019certified}, but also in random forests \cite{brophy2021machine} and clustering algorithms \cite{ginart2019making}. While some of those unlearning algorithms were only evaluated empirically, some provide strong privacy guarantees. The privacy term that is usually used in the domain of machine unlearning is that of certified removal \cite{guo2019certified}. Similar to differential privacy \cite{dwork2014algorithmic}, certified removal compares unlearning results with retraining without the target data point. Simply put, if the result of unlearning could also be obtained via retraining with a similarly high chance, then the unlearning algorithm guarantees a certain degree of privacy since the results cannot be told apart with high probability.

In this work, we will present an unlearning algorithm for sum-product networks, which is the first to our best knowledge. In contrast to Crypto-SPN \cite{treiber2020cryptospn}, which guarantees privacy-preserving inference, our unlearning algorithm preserves privacy on the model level. This means that no information about the sensitive data points can be gained, even if an attacker has white-box access to the model.

This work is structured as follows. First, we will provide all necessary foundations on sum-product networks (\Cref{sec:spn}) and certified removal (\Cref{sec:cr}). This includes sum-product networks, the training algorithm \textsc{LearnSPN}, and $\epsilon$-certified removal. In \Cref{sec:unlearnspn}, we then present some modifications of \textsc{LearnSPN} that are necessary to obtain a model that allows for certified removal, followed by our unlearning algorithm \textsc{UnlearnSPN}. Finally, we evaluate the runtime of \textsc{UnlearnSPN} experimentally in \Cref{sec:experiments} and give some directions for future research in \Cref{sec:conclusion}.

The contributions of this work can be summarized as follows:
\begin{itemize}
    \item We present a modified version of \textsc{LearnSPN} that produces sum-product networks that allow for certified data removal.
    \item We present the first unlearning algorithm for sum-product networks -- \textsc{UnlearnSPN}.
    \item We prove that \textsc{UnlearnSPN} is a 0-certified removal algorithm and, therefore, perfectly removes the influence of target data points.
    \item In our experiments, we show that \textsc{UnlearnSPN} provides a speed-up of 10-58\% compared to retraining. At the same time, the modified training algorithm only slightly increases the initial training duration.
\end{itemize}

\section{Sum-Product Networks}\label{sec:spn}
    Sum-product networks (SPNs) form a class of probabilistic graphical models that gained popularity over the last few years and were first presented by Poon and Domingos in 2011 \cite{poon2011sum}. The main idea is to represent a probability distribution employing mixing and factorizing univariate distributions. This allows the representation of more complex distributions by only using sums and products, directly leading to the definition of SPNs.

\begin{definition}[Sum-Product Network \cite{poon2011sum}]
	Let $X$ be a dataset defined over variables $V$. An SPN $\Phi$ is a rooted acyclic directed tripartite graph, where
	\begin{itemize}
		\item every leaf node represents a univariate probability distribution $P(X_{|v})$ with $v \in V$,
		\item each sum node $s$ represents a mixture of its children $P_s(X_{|scope(s)}) = \sum_{c \in children(s)} \omega_{s, c}P_c(X_{|scope(c)})$,
		\item each product node $p$ represents a factorization of its children $P_p(X_{|scope(p)}) = \prod_{c \in children(p)} P_c(X_{|scope(c)})$.
	\end{itemize}
	$\omega_{s, c}$ is a non-negative weight assigned to the edge $(s, c)$, and $\sum_{c \in children(s)} \omega_{s, c} = 1$ for each sum node $s$.
\end{definition}

The scope of a node is defined as the subset of variables the node resp. the corresponding distribution argues about. With $X_{|V'}$ we denote the dataset $X$, where only the variables $V' \subseteq V$ are considered.

The sum nodes in an SPN represent mixture models, while the product nodes represent factorizations. In order to obtain meaningful mixtures and factorizations, the scopes of the corresponding distributions must be equal for sum nodes and disjunct for product nodes. From a data perspective this means that the dataset $X$ is split into $k$ partition sets $X_1, \dots, X_k$ at sum nodes, and sliced into $m$ sets $X_{|V_1}, \dots, X_{|V_m}$ at product nodes with $\bigcap_{i \in [1, m]} V_i = \emptyset$ and $\bigcup_{i \in [1, m]} V_i = V$.
The root of an SPN then represents the complete joint distributions over all variables in $V$ and with respect to the whole dataset $X$.


    A common way to learn both the structure and the parameters of an SPN is the \textsc{LearnSPN} algorithm (\Cref{alg:learnspn}) \cite{gens2013learning}.
Here, we assume that all leaf nodes must represent univariate distributions and that the dataset will not be split any further if it is only defined over a single variable. Also, note that the clustering algorithm and the independency analysis used in \textsc{LearnSPN} can be chosen freely. There are five operations that are used to create nodes in an SPN.

\begin{algorithm}
    \begin{algorithmic}[1]
        \REQUIRE{Dataset $X$ defined over variables $V$; threshold $t$}
    	\ENSURE{SPN $\Phi$}
    	\IF{$|V| = 1$} \RETURN $\textsc{CreateLeaf}(X)$ \ENDIF
    	\IF{$\exists v \in V. \sigma^2(X_v) = 0$}
    	    \IF{$\forall v \in V. \sigma^2(X_v) = 0$} \RETURN $\textsc{NaiveFactorization}(X)$
    	    \ELSE \RETURN $\textsc{SplitUninformativeVariables}(X)$
    	    \ENDIF
        \ENDIF
        \IF{$|X| \leq t \vee (\neg clusters \wedge \neg independencies)$} \RETURN $\textsc{NaiveFactorization}(X)$ \ENDIF
        \IF{$\neg independencies$} \RETURN $\textsc{SplitData}(X)$
        \ELSE \RETURN $\textsc{SplitVariables}(X)$
        \ENDIF
    \end{algorithmic}
    \caption{\textsc{LearnSPN}}
    \label{alg:learnspn}
\end{algorithm}

\textsc{CreateLeaf} estimates an univariate distribution from $X$. The kind of distribution depends on the variable itself. Note that \textsc{CreateLeaf} will only be called, if the dataset $X$ is defined over a single variable, i.e. $|V| = 1$.

The \textsc{NaiveFactorization} creates a product node as the SPN's root and a leaf node for each variable in $V$.
There are two cases in which a naive factorization is performed. First, if all variables are uninformative, $\forall v \in V. \sigma(X_v) = 0$. Second, if there is still more than one variable present, and none is uninformative, but the number of data points $X$ is smaller than $t$. The threshold $t$ guarantees that the number of data points used to estimate the univariate distributions in the leaf nodes is large enough for reasonable estimation.

Similarly, \textsc{SplitUninformativeVariables} creates a product node as the root and a leaf node for each uninformative variable, while $X$ with respect to the remaining informative variables will be processed recursively.

\textsc{SplitData} first performs a clustering algorithm, which will result in a partition of $X$. Then, a sum node will be created as the root, and each data subset in the partition will be processed recursively. The edge weights are defined as the ratio between the number of data points in the clusters and the total number of data points.

\textsc{SplitVariables} performs a pairwise independence analysis of all variables. Afterward, the coefficients from the analysis are used to build an adjacency matrix indicating which variables are independent resp. dependent. The variables are split into connected components using the adjacency matrix. For each resulting variable subsets, the dataset $X$ is processed recursively.

Since we will often argue about the operation decision process of \textsc{LearnSPN} in the following, we illustrate it in \Cref{fig:learnspn} for easier comprehension.

\begin{figure}[h!]
    \centering
    \resizebox{0.48\textwidth}{!}{\tikzset{every picture/.style={line width=0.75pt}} 

\begin{tikzpicture}[x=0.75pt,y=0.75pt,yscale=-1,xscale=1]

\draw   (200.98,48.99) .. controls (200.97,43.46) and (205.43,38.98) .. (210.95,38.97) .. controls (216.47,38.96) and (220.96,43.42) .. (220.97,48.94) .. controls (220.98,54.46) and (216.51,58.95) .. (210.99,58.96) .. controls (205.47,58.97) and (200.99,54.51) .. (200.98,48.99) -- cycle ;
\draw   (239.98,99.99) .. controls (239.97,94.46) and (244.43,89.98) .. (249.95,89.97) .. controls (255.47,89.96) and (259.96,94.42) .. (259.97,99.94) .. controls (259.98,105.46) and (255.51,109.95) .. (249.99,109.96) .. controls (244.47,109.97) and (239.99,105.51) .. (239.98,99.99) -- cycle ;
\draw   (180.98,193.99) .. controls (180.97,188.46) and (185.43,183.98) .. (190.95,183.97) .. controls (196.47,183.96) and (200.96,188.42) .. (200.97,193.94) .. controls (200.98,199.46) and (196.51,203.95) .. (190.99,203.96) .. controls (185.47,203.97) and (180.99,199.51) .. (180.98,193.99) -- cycle ;
\draw   (428.98,173.99) .. controls (428.97,168.46) and (433.43,163.98) .. (438.95,163.97) .. controls (444.47,163.96) and (448.96,168.42) .. (448.97,173.94) .. controls (448.98,179.46) and (444.51,183.95) .. (438.99,183.96) .. controls (433.47,183.97) and (428.99,179.51) .. (428.98,173.99) -- cycle ;
\draw   (464.98,264.99) .. controls (464.97,259.46) and (469.43,254.98) .. (474.95,254.97) .. controls (480.47,254.96) and (484.96,259.42) .. (484.97,264.94) .. controls (484.98,270.46) and (480.51,274.95) .. (474.99,274.96) .. controls (469.47,274.97) and (464.99,270.51) .. (464.98,264.99) -- cycle ;
\draw    (210.99,58.96) -- (248.39,88.72) ;
\draw [shift={(249.95,89.97)}, rotate = 218.52] [color={rgb, 255:red, 0; green, 0; blue, 0 }  ][line width=0.75]    (10.93,-3.29) .. controls (6.95,-1.4) and (3.31,-0.3) .. (0,0) .. controls (3.31,0.3) and (6.95,1.4) .. (10.93,3.29)   ;
\draw    (249.99,109.96) -- (192.2,182.41) ;
\draw [shift={(190.95,183.97)}, rotate = 308.58] [color={rgb, 255:red, 0; green, 0; blue, 0 }  ][line width=0.75]    (10.93,-3.29) .. controls (6.95,-1.4) and (3.31,-0.3) .. (0,0) .. controls (3.31,0.3) and (6.95,1.4) .. (10.93,3.29)   ;
\draw    (249.99,109.96) -- (437.03,163.42) ;
\draw [shift={(438.95,163.97)}, rotate = 195.95] [color={rgb, 255:red, 0; green, 0; blue, 0 }  ][line width=0.75]    (10.93,-3.29) .. controls (6.95,-1.4) and (3.31,-0.3) .. (0,0) .. controls (3.31,0.3) and (6.95,1.4) .. (10.93,3.29)   ;
\draw    (438.99,183.96) -- (474.05,253.18) ;
\draw [shift={(474.95,254.97)}, rotate = 243.14] [color={rgb, 255:red, 0; green, 0; blue, 0 }  ][line width=0.75]    (10.93,-3.29) .. controls (6.95,-1.4) and (3.31,-0.3) .. (0,0) .. controls (3.31,0.3) and (6.95,1.4) .. (10.93,3.29)   ;
\draw    (210.99,58.96) -- (154.49,109.91) ;
\draw [shift={(153,111.25)}, rotate = 317.96] [color={rgb, 255:red, 0; green, 0; blue, 0 }  ][line width=0.75]    (10.93,-3.29) .. controls (6.95,-1.4) and (3.31,-0.3) .. (0,0) .. controls (3.31,0.3) and (6.95,1.4) .. (10.93,3.29)   ;
\draw    (190.99,203.96) -- (122.68,247.92) ;
\draw [shift={(121,249)}, rotate = 327.24] [color={rgb, 255:red, 0; green, 0; blue, 0 }  ][line width=0.75]    (10.93,-3.29) .. controls (6.95,-1.4) and (3.31,-0.3) .. (0,0) .. controls (3.31,0.3) and (6.95,1.4) .. (10.93,3.29)   ;
\draw    (190.99,203.96) -- (219.51,318.06) ;
\draw [shift={(220,320)}, rotate = 255.97] [color={rgb, 255:red, 0; green, 0; blue, 0 }  ][line width=0.75]    (10.93,-3.29) .. controls (6.95,-1.4) and (3.31,-0.3) .. (0,0) .. controls (3.31,0.3) and (6.95,1.4) .. (10.93,3.29)   ;
\draw    (438.99,183.96) -- (352.63,245.84) ;
\draw [shift={(351,247)}, rotate = 324.38] [color={rgb, 255:red, 0; green, 0; blue, 0 }  ][line width=0.75]    (10.93,-3.29) .. controls (6.95,-1.4) and (3.31,-0.3) .. (0,0) .. controls (3.31,0.3) and (6.95,1.4) .. (10.93,3.29)   ;
\draw    (474.99,274.96) -- (403.48,339.66) ;
\draw [shift={(402,341)}, rotate = 317.86] [color={rgb, 255:red, 0; green, 0; blue, 0 }  ][line width=0.75]    (10.93,-3.29) .. controls (6.95,-1.4) and (3.31,-0.3) .. (0,0) .. controls (3.31,0.3) and (6.95,1.4) .. (10.93,3.29)   ;
\draw    (474.99,274.96) -- (528.7,337.48) ;
\draw [shift={(530,339)}, rotate = 229.34] [color={rgb, 255:red, 0; green, 0; blue, 0 }  ][line width=0.75]    (10.93,-3.29) .. controls (6.95,-1.4) and (3.31,-0.3) .. (0,0) .. controls (3.31,0.3) and (6.95,1.4) .. (10.93,3.29)   ;

\draw (230,26.4) node [anchor=north west][inner sep=0.75pt]    {$|V|=1$};
\draw (34,253) node [anchor=north west][inner sep=0.75pt]   [align=left] {\textsc{NaiveFactorization}};
\draw (130,323) node [anchor=north west][inner sep=0.75pt]   [align=left] {\textsc{SplitUninformativeVariables}};
\draw (350.5,346.5) node [anchor=north west][inner sep=0.75pt]   [align=left] {\textsc{SplitData}};
\draw (505,348) node [anchor=north west][inner sep=0.75pt]   [align=left] {\textsc{SplitVariables}};
\draw (279,81.4) node [anchor=north west][inner sep=0.75pt]    {$\exists v\in V.\ \sigma ^{2}( X_{v}) =0$};
\draw (168,64) node [anchor=north west][inner sep=0.75pt]   [align=left] {1};
\draw (207,130) node [anchor=north west][inner sep=0.75pt]   [align=left] {1};
\draw (146,200) node [anchor=north west][inner sep=0.75pt]   [align=left] {1};
\draw (378,195) node [anchor=north west][inner sep=0.75pt]   [align=left] {1};
\draw (426,286) node [anchor=north west][inner sep=0.75pt]   [align=left] {1};
\draw (231,54) node [anchor=north west][inner sep=0.75pt]   [align=left] {0};
\draw (349,116) node [anchor=north west][inner sep=0.75pt]   [align=left] {0};
\draw (464,210) node [anchor=north west][inner sep=0.75pt]   [align=left] {0};
\draw (512,296) node [anchor=north west][inner sep=0.75pt]   [align=left] {0};
\draw (209,184.4) node [anchor=north west][inner sep=0.75pt]    {$\forall v\in V.\ \sigma ^{2}( X_{v}) =0$};
\draw (211,239) node [anchor=north west][inner sep=0.75pt]   [align=left] {0};
\draw (456,158.4) node [anchor=north west][inner sep=0.75pt]    {$|X| \leq t \vee (\neg independencies \wedge \neg clusters)$};
\draw (491,254.4) node [anchor=north west][inner sep=0.75pt]    {$\neg independencies$};
\draw (95,115) node [anchor=north west][inner sep=0.75pt]   [align=left] {\textsc{CreateLeaf}};
\draw (282,252) node [anchor=north west][inner sep=0.75pt]   [align=left] {\textsc{NaiveFactorization}};

\end{tikzpicture}}
    \caption{Operation decision process in \textsc{LearnSPN} (\Cref{alg:learnspn}).}
    \label{fig:learnspn}
\end{figure}
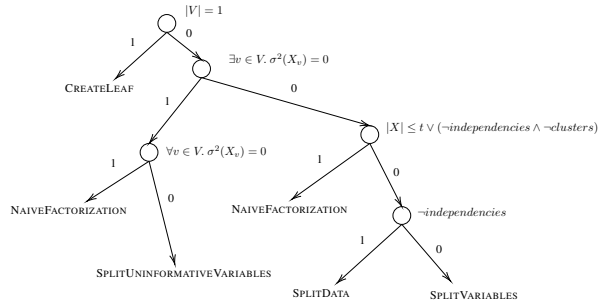

\section{Certified Removal}\label{sec:cr}
    In the domain of machine unlearning, privacy is usually described in the sense of certified removal (\Cref{def:cr}) as presented by Guo et al. \cite{guo2019certified}.

\begin{definition}[$\epsilon$-Certified Removal \cite{guo2019certified}]\label{def:cr}
    Let $\mathcal{D}$ be a data space, $\mathcal{H}$ a hypothesis space, $\mathcal{A} :\mathcal{P}(\mathcal{D}) \rightarrow \mathcal{H}$ a learning algorithm and $\mathcal{U} :\mathcal{H} \times \mathcal{D} \rightarrow \mathcal{H}$ an unlearning algorithm. $\mathcal{U}$ is called \textbf{$\epsilon$-certified removal} ($\epsilon$-CR), if and only if
		\begin{equation*}
			e^{-\epsilon} \leq \frac{P(\mathcal{U}(\mathcal{A}(X), x) \in \mathcal{T})}{P(\mathcal{A}(X \setminus \{x\}) \in \mathcal{T})} \leq e^\epsilon
		\end{equation*}
		holds $\forall \mathcal{T} \subseteq \mathcal{H}, X \subseteq \mathcal{D}, x \in X$.
\end{definition}

The intuition behind $\epsilon$-CR is that if the result of unlearning could also likely be obtained by retraining, then the unlearning algorithm successfully removes the influence of the target data point from the model. Depending on the model class, it is neither necessary nor expected that unlearning yields the exact same result as retraining. Therefore, $\epsilon$-CR argues about arbitrary hypothesis sub-spaces $\mathcal{T} \subseteq \mathcal{H}$ in which the results of unlearning and retraining may fall. In order to guarantee $\epsilon$-CR, the ratio of the chances for both results falling into $\mathcal{T}$ must be limited by $e^{-\epsilon}$ and $e^\epsilon$. The closer $\epsilon$ is to 0, the higher the privacy guarantees.


\section{Remove Data from SPNs}\label{sec:unlearnspn}
    In the following, we will describe which additions to the \textsc{LearnSPN} algorithm are necessary to enable the resulting SPN for certified data removal. Afterward, we present the revision function that maps an operation chosen during training to the operation that would have been chosen if a specific data point had not been present in the training data.
    With this being an essential part of unlearning, we finally present \textsc{UnlearnSPN} -- a 0-CR algorithm that perfectly removes the influence of a data point.

    \subsection{Certified Removal Enabled SPNs}
    In order to obtain a certified removal enabled SPN, some small additions to \textsc{LearnSPN} must be made. During training, we define a state for each node that contains all the information necessary for efficiently performing unlearning afterward. More precisely, for each node $n$ we define a state as an 8-tuple $s = (sc$, $op$, $X$, $N$, $independencies$, $clusters$, $exist\_uninformative$, $ all\_uninformative)$,
where $sc$ is the scope of $n$, $op$ the operation chosen to create the sub-SPN rooted at $n$, $X$ the dataset of size $N$ used for training, $independencies$ and $clusters$ indicate whether there exist independent variables or data clusters, and $exist\_uninformative$ and $all\_uninformative$ indicate whether there exist uninformative variables or if all variables are uninformative. For sum and product nodes, we additionally store an instance of the clustering algorithm and the independency analysis, respectively. For leaf nodes, we also store additional information that will allow us to update the univariate distribution quickly.
Storing a state for each node in an SPN yields an additional memory consumption of $\mathcal{O}(|V| + |X|)$ per node. The additional memory consumption induced by storing the clustering and the independency analysis depends on the choice of the splitting algorithms. Since we keep track of the operations used during training, it is important that the SPN must not be pruned afterward; otherwise, the operation used for creating a sub-SPN becomes ambiguous.
Except for the above additions, no further changes to \textsc{LearnSPN} are necessary.

Also, note that even though arbitrary clustering algorithms and independency analyses can be used, it is highly recommended to use algorithms that allow 0-certified data removal or can be updated in a reasonable amount of time to reduce the cost of removing data from the SPN.
    \subsection{Revising the Choice of Operation}
    Next, we introduce the revision function $rev$ (\Cref{def:revision}), which will be an essential part of our unlearning algorithm.

\allowdisplaybreaks

\begin{definition}\label{def:revision}
    Let $X$ be a dataset defined over variables $V$. Let $x \in X$ be a data point and $X' = X \setminus \{x\}$. Let \textsc{CreateLeaf} (CL), \textsc{NaiveFactorization} (NF), \textsc{SplitUninformativeVariables} (SU), \textsc{SplitData} (SD) and \textsc{SplitVariables} (SV) be the operations used in \textsc{LearnSPN} (\Cref{alg:learnspn}). 
    The revision function $rev$ takes an operation and maps it to another operation in the following manner:
    \begin{align*}
    	&\text{CL} \mapsto \text{CL}\\
    	&\text{NF} \mapsto \begin{cases}
    		\text{NF}, &\text{if } \forall v \in V. \sigma^2(X_v) = 0 \\
    		\text{NF}, &\text{if } \forall v \in V. \sigma^2(X'_v) = 0 \\
    		\text{SU}, &\text{if } \exists v \in V.\sigma^2(X'_v) = 0 ~\wedge \not\exists v \in V.\sigma^2(X_v) = 0 \\
    		\text{NF}, &\text{if } |X| \leq t \vee (\neg clusters \wedge \neg independencies) \\
    		\text{SD}, &\text{if } \neg independencies \\
    		\text{SV}, &\text{otherwise}
    	\end{cases} \\
    	&\text{SD} \mapsto \begin{cases}
    		\text{SU}, &\text{if } \exists v \in V. \sigma^2(X'_v) = 0 ~ \wedge \not\forall v \in V. \sigma^2(X'_v) = 0\\
    		\text{NF}, &\text{if } \forall v \in V. \sigma^2(X'_v) = 0 \\
    		\text{NF}, &\text{if} \not\exists v \in V. \sigma^2(X'_v) = 0 ~ \wedge |X'| \leq t\\
    		\text{NF}, &\text{if} \not\exists v \in V. \sigma^2(X'_v) = 0 ~ \wedge \neg clusters\\
    		&\wedge \neg independencies\\
    		\text{SV}, &\text{if} \not\exists v \in V.\sigma^2(X'_v) = 0~ \wedge |X'| > t \\
    		&\wedge independencies \\
    		\text{SD}, &\text{otherwise} \\
    	\end{cases}\\
    	&\text{SV} \mapsto \begin{cases}
    		\text{SU}, &\text{if } \exists v \in V. \sigma^2(X'_v) = 0 ~ \wedge \not\forall v \in V. \sigma^2(X'_v) = 0\\
    		\text{NF}, &\text{if } \forall v \in V. \sigma^2(X'_v) = 0 \\
    		\text{NF}, &\text{if} \not\exists v \in V. \sigma^2(X'_v) = 0 ~ \wedge |X'| \leq t\\
    		\text{NF}, &\text{if} \not\exists v \in V. \sigma^2(X'_v) = 0 ~ \wedge \neg clusters\\
    		&\wedge \neg independencies \\
    		\text{SD}, &\text{if} \not\exists v \in V.\sigma^2(X'_v) = 0~ \wedge |X'| > t \\
    		&\wedge ~ clusters \wedge \neg independencies \\
    		\text{SV}, &\text{otherwise} \\
    	\end{cases}\\
    	&\text{SU} \mapsto \begin{cases}
    	    \text{SU}, &\text{if } \exists v \in V. \sigma^2(X'_v) > 0 \\
    		\text{NF}, &\text{otherwise} \\
    	\end{cases}
    \end{align*}
    $\sigma^2$ denotes the variance and $t$ is the data threshold used in \textsc{LearnSPN}.
\end{definition}

Note that the order of cases in \Cref{def:revision} matters. For each case, we expect that the conditions for all cases above are falsified. In \Cref{th:revision}, we show that $rev$ is sound and complete, i.e., it exactly corresponds to the change in the chosen operation in \textsc{LearnSPN} when a certain data point would not have been present during training.

\begin{theorem}\label{th:revision}
    Let $X$ be a dataset defined over variables $V$, $x \in X$ a data point, $X' = X \setminus \{x\}$ and $s$ be a random seed.
    The revision function $rev$ maps an operation $op_{old}$ to $op_{new}$ such that
    \begin{equation}\label{eq:th_revision}
        \begin{aligned}
            rev(op_{old}) = op_{new} \Leftrightarrow ~ \textsc{LearnSPN}(X; s) \text{ chooses } op_{old} \wedge\textsc{LearnSPN}(X'; s) \text{ chooses } op_{new}.
        \end{aligned}
    \end{equation}
\end{theorem}

\begin{proof}[Proof of \Cref{th:revision}]
    Since \textsc{LearnSPN} is a randomized algorithm, we assume \textsc{LearnSPN}($X$; $s$) and \textsc{LeanrSPN}($X'$; $s$) to be seeded with the same random seed $s$. This way \textsc{LearnSPN} becomes a deterministic algorithm. For reasons of simplicity we omit the random seed $s$ in the further course.
    In the following, we prove the equivalence stated in \Cref{th:revision} for each input operation separately. To follow the reasoning in this proof more easily, we refer to the operation decision process illustrated in \Cref{fig:learnspn}.
    

    \textbf{Case \textsc{Create Leaf} (CL):}
    \textsc{LearnSPN}($X$) chooses \textsc{CL}, if and only if $|V| = 1$. Removing an arbitrary data point $x$ from $X$ has no influence on the number of variables $|V|$. Thus, $\textsc{LearnSPN}(X) \text{ chooses } \textsc{CL} \Leftrightarrow ~ |V_X| = 1 \Rightarrow |V_{X'}| = 1 \Leftrightarrow \textsc{LearnSPN}(X')\text{ chooses }\textsc{CL}.$
    If $op_{old} = \textsc{CL}$, then removing a data point $x$ from $X$ does not change the chosen operation, i.e. $op_{old} = \textsc{CL} = op_{new}$. Since $rev(\textsc{CL}) = \textsc{CL}$, the equivalence in \Cref{eq:th_revision} holds.


    \textbf{Case \textsc{NaiveFactorization} (NF):}
    There are two cases in which \textsc{LearnSPN}$(X)$ chooses \textsc{NF} (see \Cref{fig:learnspn}): 1) $|V| > 1 \wedge \forall v \in V. \sigma^2(X_v) = 0$, and 2) $|V| > 1 \wedge \not\exists v \in V. \sigma^2(X_v) = 0 \wedge |X| \leq t \vee ~ (\neg clusters \wedge \neg independencies)$.

    \textbf{If \textsc{NF} was chosen due to case 1):}
    The number of variables is unaffected by removing $x$ from $X$, i.e. $|V_X| > 1 \Rightarrow |V_{X'} > 1|$.
    We know that all points must be equal if the variance is 0: $\forall v \in V. \sigma^2(X_v) = 0 \Leftrightarrow \forall v \in V. \frac{1}{|X|} \sum_{p \in X_v} (p - \mu_v)^2 = 0
    \Leftrightarrow \forall v \in V. \forall p \in X_v. (p - \mu_v)^2 = 0
    \Leftrightarrow \forall v \in V. \forall p, q \in X_v. p = q$,
    where $\mu_v$ is the mean value of $X_v$. After removing $x$ from $X$, we know that the remaining data points are still all equal. Therefore, $\forall v \in V. \sigma^2(X_v) = 0 \Leftrightarrow \forall v \in V. \forall p, q \in X_v. p = q \Rightarrow \forall v \in V. \forall p, q \in X'_v. p = q \Leftrightarrow \forall v \in V. \sigma^2(X'_v) = 0$.
    So, if 1) holds, then both \textsc{LearnSPN}($X$) and \textsc{LearnSPN}($X'$) choose \textsc{NF}.
    This mapping is captured by the first case of $rev$(\textsc{NF}) (see \Cref{def:revision}).

    \textbf{If \textsc{NF} was chosen due to case 2):}
    Again, we know that the number of variables $|V|$ is unaffected. This is the case if all data points except for $x$ are identical w.r.t. a non-empty variable subset $V' \subseteq V$.
    
    If $V' = V$, then $\not\exists v \in V. \sigma^2(X_v) = 0, \text{but }\forall v \in V'. \sigma^2(X'_v) = 0$.
    Nonetheless, \textsc{LearnSPN}($X'$) would choose \textsc{NF} since falsifying $\not\exists v \in V. \sigma^2(X_v) = 0$ is equivalent to accepting the second condition of 1) for \textsc{LearnSPN}($X'$). So, if 2) holds for $X$ and $\forall v \in V. \sigma^2(X'_v) = 0$ both \textsc{LearnSPN}($X$) and \textsc{LearnSPN}($X'$) choose \textsc{NF}. This mapping is captured by the second case of $rev$(\textsc{NF}).
    
    If $V' \neq V$, then there exist some variables that become uninformative when removing $x$: $\not\exists v \in V. \sigma^2(X_v) = 0, \text{but }\exists v \in V'. \sigma^2(X'_v) = 0$.
    However, there are also variables that are still informative since $V' \neq V$. In this case \textsc{LearnSPN}($X'$) would choose to split the uninformative variables rather than performing a naive factorization (see \Cref{fig:learnspn}). Note that this decision is independent of $|X| \leq t \vee (\neg clusters \wedge \neg independencies)$. This mapping is captured by the third case of $rev$(\textsc{NF}).

    Finally, it is also possible that all variables stay informative. In this case, the chosen operation only depends on $|X| \leq t$ and the existence of clusters and independencies. If $|X| \leq t \vee ~ (\neg clusters \wedge \neg independencies)$ holds due to $|X| \leq t$, then we know that it also holds after removing $x$ since $|X'| < |X| \leq t$. The chosen operation would still be \textsc{NF}. This is also the case if $|X| = t+1$. If it holds due to $\neg clusters \wedge \neg independencies$, removing $x$ might induce the existence of clusters or independent variables. Therefore, instead of a naive factorization, \textsc{SD} or \textsc{SV} is chosen by \textsc{LearnSPN}. The three cases above are captured by the last three cases of $rev(\textsc{NF})$.
    
    Since $rev$(\textsc{NF}) correctly captures all of the above four cases, we know that the equivalence stated in \Cref{eq:th_revision} holds in this case.
    

    \textbf{Case \textsc{SplitUninformativeVariables} (SU):}
    \textsc{LearnSPN}($X$) chooses to split uninformative variables, if and only if $|V| > 1 \wedge \exists v \in V. \sigma^2(X_v) = 0 ~ \wedge \not\forall v \in V. \sigma^2(X_v) = 0$.
    Again the number of variables is unaffected. Furthermore, we know that  $\exists v \in V. \sigma^2(X_v) = 0$ also holds after removing $x$, since for those variables with variance 0 all data points must be equal. Thus, we only have to distinguish the two cases in which $\not\forall v \in V. \sigma^2(X_v) = 0$ is either satisfied or falsified after removing $x$: 1) $\not\forall v \in V. \sigma^2(X'_v) = 0$, and 2) $\forall v \in V. \sigma^2(X'_v) = 0$.
    
    If 1) holds, we know that both \textsc{LearnSPN}($X$) and \textsc{LearnSPN}($X'$) must choose to split uninformative variables, since there exist uninformative variables, i.e. $\exists v \in V. \sigma^2(X_v) = 0$, but not all variables are uninformative (see \Cref{fig:learnspn}). In case that 2) holds, \textsc{LearnSPN}($X'$) would choose \textsc{NF} instead, since no informative variables would be left. $rev$(\textsc{SU}) captures both cases and therefore, we know that the equivalence in \Cref{eq:th_revision} holds in this case.
    
    
    \textbf{Case \textsc{SplitData} (SD):}
    \textsc{LearnSPN}($X$) chooses to split data, if and only if $|V| > 1 \wedge \not\exists v \in V.\sigma^2(X_v) = 0	\wedge |X| > t \wedge clusters \wedge \neg independencies$.

    The number of variables is again unaffected. As already argued above, removing $x$ might yield a variance of 0 for some variables. Thus, it is possible that $\exists v \in V.\sigma^2(X'_v) = 0$. This implies that \textsc{LearnSPN}($X'$) would either choose a naive factorization or splitting uninformative variables, depending on whether all variables become uninformative. Both cases are covered by the first two cases of $rev$(\textsc{SD}).

    In the following, we assume $\not\exists v \in V.\sigma^2(X_v) = 0$ to hold after removing $x$. The condition $|X| > t$ can be falsified after removing $x$, if $|X| = t + 1 > t = |X'|$. In this case \textsc{LearnSPN}($X'$) would choose \textsc{NF} instead of \textsc{SD}, because $|X'|$ goes below the minimum number of required data points $t$. This is captured in the third case of $rev$(\textsc{SD}).
    
    If the $|X| > t$ still holds after removing $x$, it can still be the case that the $clusters$ condition is falsified afterward. This happens if removing $x$ results in a clustering, where all remaining data points belong to the same cluster. At this point, we have to further distinguish between the two cases where $\neg independencies$ is satisfied or falsified. If $\neg independencies$ is still satisfied after removing $x$, then we know that \textsc{LearnSPN}($X'$) would choose \textsc{NF}, since there exists more than one variable, no variables are uninformative, the dataset consists of at least $t+1$ data points and there neither exist independent variables nor clusters in the data. However, if removing $x$ induces the existence of independencies, then LearnSPN($X'$) would choose \textsc{SV} instead of \textsc{SD}. This will also be the case if $clusters$ is still satisfied after removing $x$. These three cases are all covered by the third to fifth case of $rev(\textsc{SD})$.
    
    If there still exist clusters after removing $x$, then the decision of \textsc{LearnSPN}($X'$) only depends on the existence of independencies. If there exist any independent variables, i.e. $independencies$ holds, then \textsc{LearnSPN}($X'$) would choose \textsc{SV} over \textsc{SD}. Otherwise, \textsc{LearnSPN}($X'$) would not revise the operation and stick with \textsc{SD}. This is captured by the last two cases in $rev$(\textsc{SD}), which therefore satisfies the equivalence in \Cref{eq:th_revision} in case of $op_{old} = \textsc{SD}$.
    

    \textbf{Case \textsc{Split Variables} (SV):}
    \textsc{LearnSPN}($X$) chooses to split variables, if and only if $|V| > 1 \wedge \not\exists v \in V.\sigma^2(X_v) = 0 \wedge |X| > t \wedge independencies$.
    
    The first three cases of $rev$(\textsc{SV}) are equivalent to $rev$(\textsc{SD}) (see above).
    In case that removing $x$ also removes independencies such that $\neg independencies$ holds, the decision of \textsc{LearnSPN}($X'$) depends on the existence of clusters. If no clusters exist, i.e. $\neg clusters$ holds, then \textsc{LearnSPN}($X'$) chooses \textsc{NF} (see \Cref{fig:learnspn}). On the other hand, \textsc{LearnSPN}($X'$) would choose \textsc{SD} over \textsc{SV}.
    If $independencies$ still holds after removing $x$, then \textsc{LearnSPN}($X'$) would also choose \textsc{SV}. This is captured by the last three cases of $rev$(\textsc{SV}), which therefore satisfies the equivalence in \Cref{eq:th_revision} in case of $op_{old} = \textsc{SV}$.\\\\
    In conclusion, we showed that for each of the five operations, the revision function $rev$ correctly captures all changes in choosing an operation in \textsc{LearnSPN} when removing a data point $x$ from the dataset $X$. Therefore \Cref{th:revision} holds in general.
\end{proof}

    \subsection{UnlearnSPN}
    Finally, we introduce \textsc{UnlearnSPN} -- an unlearning algorithm for certified removal enabled SPNs (\Cref{alg:unlearnspn}).

\begin{algorithm}
    \begin{algorithmic}[1]
        \REQUIRE{SPN $\Phi$, data point $x \in X$}
    	\ENSURE{SPN $\Phi$ without the influence of $x$}
    	\STATE $st \gets \textsc{State}(\Phi)$
    	\IF{$x \notin st.data$} \RETURN $\Phi$ \ENDIF
    	\STATE $op_{old} \gets st.operation$
    	\STATE $op_{new} \gets \textsc{ReviseOperation}(st, x)$
    	\IF{($op_{old} = $ \textsc{SD} $\wedge$ $op_{new} = $ \textsc{SV}) \\ $\vee$($op_{old} = $ \textsc{SV} $\wedge$ $op_{new} = $ \textsc{SD}) \\ $\vee$ ($op_{old} \in \{$\textsc{SV}, \textsc{SD}$\}$ $\wedge$ $op_{new} = $ \textsc{SU})\\ $\vee (op_{old} = \textsc{NF} \wedge op_{new} \in \{ \textsc{SD}, \textsc{SV} \})$}
    	\RETURN \textsc{LearnSPN}($st.data \setminus \{x\}$)
    	\ENDIF
    
    	\IF{$op_{old} \neq op_{new}$} \RETURN \textsc{NaiveFactorization}$(st.data \setminus \{x\})$ \ENDIF
        
    	\SWITCH{$op_{old}$}
        \STATE \textit{/* call unlearning algorithm corresponding to $op_{old}$ */}
    	\ENDSWITCH
    \end{algorithmic}
    \caption{\textsc{UnlearnSPN}}
    \label{alg:unlearnspn}
\end{algorithm}

The key idea behind \textsc{UnlearnSPN} is that we reevaluate the decisions made during training by using our revision function from \Cref{def:revision}. If the chosen operation does not change when removing the target data point $x$, we only have to update the state and parameters of the root node and proceed with the child nodes. If the chosen operation changes, we have to retrain the sub-SPN rooted at the current node. If $x$ is not present in the data of the current sub-SPN, then nothing has to be done. The conditions in lines 6-9 correspond to the cases in $rev$, where the input operation differs from the output operation. Note that lines 8-9 handle all cases where the new operation is a naive factorization. Thus, instead of calling \textsc{LearnSPN} we directly perform the naive factorization. If the old and new operation are the same, the updates that must be made depend on the operation that created the current sub-SPN.
We refer to these updates as \textsc{UnlearnCreateLeaf}, \textsc{UnlearnSplitData} and so forth (\Cref{alg:unlearnnaivefactorization,alg:unlearncreateleaf,alg:unlearnsplituninformative,alg:unlearnsplitdata,alg:unlearnsplitvariables}).


 \begin{algorithm}
     \begin{algorithmic}[1]
         \REQUIRE{SPN $\Phi$, data point $x \in X$}
     	\ENSURE{SPN $\Phi$ without the influence of $x$}
     	\STATE $st \gets $\textsc{State}($\Phi$)
     	\STATE $st.num\_data \gets st.num\_data - 1$
     	\STATE $st.data \gets st.data \setminus \{x\}$
     	\FORALL{$child \in st.children$} \STATE \textsc{UnlearnCreateLeaf}$(child, x)$ \ENDFOR
     	\IF {$st.all\_uninformative$} \RETURN $\Phi$ \ENDIF
     	\STATE $\sigma^2 \gets \textsc{Variances}(st.data)$
     	\IF{$\sigma^2_V = 0$ for all $V \in st.scope$}
     	\STATE $st.all\_uninformative \gets true$
     	\STATE $st.exist\_uninformative \gets true$
         \ENDIF
     	\RETURN $\Phi$
     \end{algorithmic}
 	\caption{\textsc{UnlearnNaiveFactorization}}
 	\label{alg:unlearnnaivefactorization}
 \end{algorithm}

 \begin{algorithm}
     \begin{algorithmic}
     	\REQUIRE{SPN $\Phi$, data point $x \in X$}
     	\ENSURE{SPN $\Phi$ without the influence of $x$}
     	\STATE $st \gets $\textsc{State}($\Phi$)
     	\STATE $st.num\_data \gets st.num\_data - 1$
     	\STATE $st.data \gets st.data \setminus \{x\}$
     	\STATE \textsc{UpdateUnivariateDistribution}($\Phi$, $x$)
     	\RETURN $\Phi$
     \end{algorithmic}
 	\caption{\textsc{UnlearnCreateLeaf}}
 	\label{alg:unlearncreateleaf}
 \end{algorithm}

 \begin{algorithm}
	\begin{algorithmic}
		\REQUIRE{SPN $\Phi$, data point $x \in X$}
		\ENSURE{SPN $\Phi$ without the influence of $x$}
		\STATE $st \gets $\textsc{State}($\Phi$)\;
		\STATE $st.num\_data \gets st.num\_data - 1$\;
		\STATE $st.data \gets st.data \setminus \{x\}$\;
		\FORALL{$child \in st.children$ where $child$ is a leaf node} \STATE \textsc{UnlearnCreateLeaf}$(child, x)$ \ENDFOR
		\STATE $\Psi \gets $ non-leaf child of $\Phi$
		\STATE $\sigma^2 \gets \textsc{Variances}(st.data)$
		\IF{$\sigma^2_V = 0$ for any $V \in \textnormal{State}(\Psi).scope$}
		\FORALL{$V \in \textsc{State}(\Psi).scope \text{ with } \sigma^2_V = 0$}
		\STATE $leaf \gets \textsc{CreateLeaf}(st.data, V)$
		\STATE \textsc{Append}($\Phi.children$, $leaf$)
		\ENDFOR
		\STATE{$\Psi \gets$ \textsc{LearnSPN}($st.data$ with remaining scope)}
		\ELSE
		\STATE $\Psi \gets$ \textsc{UnlearnSPN}($\Psi$, $x$)
		\ENDIF
		\RETURN $\Phi$
	\end{algorithmic}
	\caption{\textsc{UnlearnSplitUninformative}}
	\label{alg:unlearnsplituninformative}
\end{algorithm}

\begin{algorithm}
	\caption{\textsc{UnlearnSplitData}}
	\begin{algorithmic}
		\REQUIRE{SPN $\Phi$, data point $x \in X$}
		\ENSURE{SPN $\Phi$ without the influence of $x$}
		\STATE $st \gets $ \textsc{State}($\Phi$)
		\STATE $clustering_{old} \gets st.clustering$
		\STATE $clustering_{new} \gets $ \textsc{Remove}($clustering_{old}$, $x$)
		\IF{$clustering_{old} \neq clustering_{new}$} \RETURN \textsc{SplitData}($st.data \setminus \{x\}$) \ENDIF
		\STATE $st.num\_data \gets st.num\_data - 1$
		\STATE $st.data \gets st.data \setminus \{x\}$
		\STATE \textsc{UpdateWeights}($\Phi$, $x$)
		\STATE $\Psi \gets child \in \Phi.children$ with $x \in $ \textsc{State}$(child).data$
		\STATE $\Psi \gets$ \textsc{UnlearnSPN}($\Psi$, $x$)
		\RETURN $\Phi$
	\end{algorithmic}
	\label{alg:unlearnsplitdata}
\end{algorithm}

 \begin{algorithm}
	\caption{\textsc{UnlearnSplitVariables}}
	\begin{algorithmic}
		\REQUIRE{SPN $\Phi$, data point $x \in X$}
		\ENSURE{SPN $\Phi$ without the influence of $x$}
		\STATE $st \gets $ \textsc{State}($\Phi$)\;
		\STATE $variable\_split_{old} \gets st.variable\_split$\;
		\STATE $variable\_split_{new} \gets$ \textsc{Remove}($variable\_split_{old}$, $x$)\;
		\IF{$variable\_split_{old} \neq variable\_split_{new}$} \RETURN \textsc{SplitVariables}($st.data \setminus \{x\}$) \ENDIF
		\STATE $st.num\_data \gets st.num\_data - 1$\;
		\STATE $st.data \gets st.data \setminus \{x\}$\;
		\FORALL{$child \in \Phi.children$} \STATE $child \gets $\textsc{UnlearnSPN}($child$, $x$) \ENDFOR
		\RETURN $\Phi$
	\end{algorithmic}
	\label{alg:unlearnsplitvariables}
\end{algorithm}

\textsc{UnlearnNaiveFactorization} updates a node that was created with \textsc{NF}. First, the data point is removed from the data set. Then \textsc{UnlearnCreateLeaf} is called for all child nodes. Next, we have to ensure that the indicators for uninformative variables are also correctly updated. If all variables were already uninformative, we do not have to make any changes since removing data points cannot make variables informative again. Otherwise, we have to check if any or all variables have a variance of 0 and update the indicators accordingly.

\textsc{UnlearnCreateLeaf} updates a leaf node by simply removing $x$ and updating the univariate distribution such that it corresponds to the dataset without $x$.

\textsc{UnlearnSplitUninformative} updates a node that was created with \textsc{SU}. Again, we have to remove $x$ from the dataset first. For \textsc{SU} we know that all but one child are leaf nodes. All leaf nodes are updated by calling \textsc{UnlearnCreateLeaf}. Next, we have to check if removing $x$ induced further uninformative variables. If this is the case, we create new leaf nodes, estimate the univariate distributions and add them as new child nodes. We have to retrain the corresponding sub-SPN from scratch for the remaining informative variables by calling \textsc{LearnSPN}. In case no new uninformative variables are introduced when removing $x$, we process the non-leaf child by calling \textsc{UnlearnSPN} recursively.

\textsc{UnlearnSplitData} is the only algorithm that is used for updating sum nodes. Here, we remove $x$ from the clustering by either recomputing the clustering with the same initialization on the remaining data or by actually removing it if the clustering allows certified removal. If the clusters have changed due to removing $x$, we have to retrain the sub-SPN starting with a \textsc{SD} operation. Otherwise, we remove $x$ from the state, update the edge weights accordingly, and call \textsc{UnlearnSPN} for the child node that was trained on the data subset containing $x$.

\textsc{UnlearnSplitVariables} updates nodes that split variables by actually using the \textsc{SV} operation. Here, we remove $x$ from the independency analysis by either recomputing or updating if possible. If the independencies are affected by the data removal, we retrain the sub-SPN from scratch, starting with \textsc{SV}, since the variables split has changed. If the independencies are unaffected, we must remove $x$ from the state and call \textsc{UnlearnSPN} recursively for all child nodes.

Next, we show in three steps that \textsc{UnlearnSPN} is a $\epsilon$ certified removal algorithm with $\epsilon = 0$. For this we separately handle the cases where the operations chosen during training are preserved (\Cref{prop:unlearn=retraining}) or changed (\Cref{th:unlearn=retraining}) by \textsc{UnlearnSPN}.

\begin{proposition}\label{prop:unlearn=retraining}
    Let $X$ be a dataset, $x \in X$ a data point, $X' = X \setminus \{x\}$, $s$ a random seed, $\Phi=$\textsc{ LearnSPN}(X; s) an SPN and $\mathcal{T} \subseteq \mathcal{H}$ a hypothesis sub-space. If $\textsc{LearnSPN}(X; s)$ chooses operation $op$ and $\textsc{LearnSPN}(X'; s)$ chooses operation $op'$ with $op \neq op'$, then
    \begin{equation}
        P(\textsc{UnlearnSPN}(\Phi, x) \in \mathcal{T}) = P(\textsc{LearnSPN}(X') \in \mathcal{T}).
    \end{equation}
\end{proposition}

\begin{proof}[Proof of \Cref{prop:unlearn=retraining}]
From \Cref{th:revision}, we know that the revision function $rev$ is equivalent to the change of the chosen operation when removing a data point $x$, i.e. $rev(op) = op'$. For \textsc{UnlearnSPN}($\Phi, x$) we know that if $op \neq rev(op) = op'$, the SPN will be retrained from scratch calling \textsc{LearnSPN}($X'$) (see \Cref{alg:unlearnspn} lines 5-9). Note that in lines 8-9, we explicitly perform a naive factorization, since we know that the new operation that would be performed when calling \textsc{LearnSPN}($X'$) would be a naive factorization anyway. Therefore, we see that in all cases where $op \neq rev(op) = op'$ \textsc{UnlearnSPN} just calls \textsc{LearnSPN} and returns its result, which concludes that \textsc{UnlearnSPN}($\Phi, x$) = \textsc{LearnSPN}($X'$). This implies $P(\textsc{UnlearnSPN}(\Phi, x) \in \mathcal{T}) = P(\textsc{LearnSPN}(X') \in \mathcal{T})$ for any hypothesis sub-space $\mathcal{T}$.
\end{proof}

In \Cref{prop:unlearn=retraining}, we proved that updating a model via \textsc{UnlearnSPN} will result in the same SPN as retraining it from scratch on the remaining data if the chosen operations are different. For the retraining, we assume that the same random seed is used as in the original training; otherwise, the results would not be comparable.

In \Cref{th:unlearn=retraining}, we utilize our result from \Cref{prop:unlearn=retraining} and show that the result obtained via \textsc{UnlearnSPN} also could have been obtained via retraining with an equal chance in general.

\begin{theorem}\label{th:unlearn=retraining}
Let $X$ be a dataset, $x \in X$ a data point, $X' = X \setminus \{x\}$, $\Phi = $\textsc{ LearnSPN}(X) an SPN and $\mathcal{T} \subseteq \mathcal{H}$ a hypothesis sub-space, then $P(\textsc{UnlearnSPN}(\Phi, x) \in \mathcal{T}) = P(\textsc{LearnSPN}(X') \in \mathcal{T})$.
\end{theorem}

\begin{proof}[Proof of \Cref{th:unlearn=retraining}]
From \Cref{prop:unlearn=retraining} we already know that $P(\textsc{UnlearnSPN}(\Phi, x) \in \mathcal{T}) = P(\textsc{LearnSPN}(X') \in \mathcal{T})$ holds, if $op \neq rev(op) = op'$, where $op$ is the operation chosen by $\textsc{LearnSPN}(X; s)$ and $op'$ the operation chosen by $\textsc{LearnSPN}(X'; s)$ for random seed $s$. Thus, we only have to prove the equation in \Cref{th:unlearn=retraining} in case of $op = op'$.
We prove this equality via structural induction over the structure of $\Phi$.

\textbf{Base case:} $\Phi$ only consists of a single node. This node must then be a leaf node and represent an univariate distribution over $X$. From $rev$(\textsc{CL}) = \textsc{CL}, we know that if \textsc{LearnSPN}($X$) creates a leaf node, so does \textsc{LearnSPN}($X'$). \textsc{UnlearnSPN} updates the univariate distribution of the leaf node such that $P(X) \mapsto P(X')$. Therefore, we know that \textsc{UnlearnSPN}($\Phi$, $x$) = \textsc{LearnSPN}($X'$), which implies $P(\textsc{UnlearnSPN}(\Phi, x) \in \mathcal{T}) = P(\textsc{LearnSPN}(X') \in \mathcal{T})$.
\\\\
\textbf{Induction hypothesis:} Let $\Phi$ be an SPN with subSPNs $\Phi_1, ..., \Phi_k$. $P(\textsc{UnlearnSPN}(\Phi_i, x) \in \mathcal{T})$ = $P(\textsc{LearnSPN}(X_i') \in \mathcal{T})$ holds for all $i \in [1, k]$, where $X_i'$ corresponds to the remaining data used for training $\Phi_i$.
\\\\
\textbf{Induction step:} In the following, we distinguish the four remaining operations that are used to build an SPN.

\textbf{Case \textsc{NaiveFactorization}:}
In case of a naive factorization, we know that all children must be leaf nodes. \textsc{UnlearnSPN} preserves the root and processes all children recursively. Since all children are leafs, we know from our base case that \textsc{UnlearnSPN} yields the same result as retraining. In conclusion, \textsc{UnlearnSPN}($\Phi$, $x$) = \textsc{LearnSPN($X'$)}, which implies $P(\textsc{UnlearnSPN}(\Phi, x) \in \mathcal{T}) = P(\textsc{LearnSPN}(X') \in \mathcal{T})$.

\textbf{Case \textsc{SplitUninformativeVariables}:}
In case of splitting uninformative variables, the SPN will consist of a product node as a root, where all but one child are leaf nodes. The non-leaf sub-SPN $\Psi$ is arbitrary.
\textsc{UnlearnSPN} preserves the product root node. We know that if a variable is uninformative in $\Phi$ it will still be uninformative after removing $x$. Therefore, \textsc{UnlearnSPN} preserves all leafs and updates them according to the base case. For the non-leaf sub-SPN we further distinguish two cases.

First, all informative variables stay informative. This means that the variable split is the same for \textsc{LearnSPN($X$)} and \textsc{LearnSPN($X'$)}. \textsc{UnlearnSPN($\Phi$, $x$)} will preserve the split as well. In this case \textsc{UnlearnSPN} processes the non-leaf sub-SPN $\Psi$ recursively. By the induction hypothesis, we assume that \textsc{UnlearnSPN}($\Psi$, $x$) = \textsc{LearnSPN}($X'_\Psi$), where $X'_\Psi$ corresponds to the remaining training data $X'$ in the remaining scope of the sub-SPN $\Psi$.

Second, we consider the case where some informative variables become uninformative. This adds $m < |scope(\Psi)| - 1$ new leaf nodes to the root. The corresponding distributions are estimated from scratch. For the remaining variables \textsc{UnlearnSPN} calls \textsc{LearnSPN}.
In conclusion, for both cases (split changes / is preserved), we can conclude that $P(\textsc{UnlearnSPN}(\Phi, x) \in \mathcal{T}) = P(\textsc{LearnSPN}(X') \in \mathcal{T})$ holds.

\textbf{Case \textsc{SplitData}:}
First, \textsc{UnlearnSPN} removes $x$ from the clustering. This is either done via recomputing the clustering with the same initialization or by updating it if possible.

If removing $x$ preserves the clusters, \textsc{UnlearnSPN} only adapts the weights of the edges and processes the sub-SPNs recursively. Since $x$ will only be part of one cluster, \textsc{UnlearnSPN} is only called for the sub-SPN $\Psi$ corresponding to the subset containing $x$. By the induction hypothesis, we assume that $P(\textsc{UnlearnSPN}(\Psi, x) \in \mathcal{T}) = P(\textsc{LearnSPN}(X'_\Psi) \in \mathcal{T})$, where $X'_\Psi$ corresponds to the remaining data used for training $\Psi$.

If removing $x$ changes the clusters, \textsc{UnlearnSPN} retrains the SPN from scratch on $X'$ starting with a split data operation.
In conclusion, $P(\textsc{UnlearnSPN}(\Phi, x) \in \mathcal{T}) = P(\textsc{LearnSPN}(X') \in \mathcal{T})$ in all cases.

\textbf{Case \textsc{Split Variables}:}
Let $V_1, ..., V_k$ be the variable partition and $\Phi_1, ..., \Phi_k$ the corresponding sub-SPNs. First, \textsc{UnlearnSPN} recomputes or updates the independencies if possible. If removing $x$ has no influence on the matrix, then \textsc{UnlearnSPN} processes all sub-SPNs recursively. By the induction hypothesis, we assume that \textsc{UnlearnSPN}($\Phi_i$, $x$) = \textsc{LearnSPN}($X'_i$), for all $i \in [1, k]$, where $X'_i$ corresponds to the remaining data in the scope of $\Phi_i$. If removing $x$ changes the adjacency matrix, then \textsc{UnlearnSPN} retrains the SPN from scratch starting with the split variable operation. In all cases $P(\textsc{UnlearnSPN}(\Phi, x) \in \mathcal{T}) = P(\textsc{LearnSPN}(X') \in \mathcal{T})$.
\\\\
\textbf{Conclusion:} If $rev(op) = op$, then $P(\textsc{UnlearnSPN}(\Phi, x) \in \mathcal{T}) = P(\textsc{LearnSPN}(X') \in \mathcal{T})$.
\\\\
The proof via structural induction for $op \neq rev(op) = op'$, together with \Cref{prop:unlearn=retraining} conclude that $P(\textsc{UnlearnSPN}(\Phi, x) \in \mathcal{T}) = P(\textsc{LearnSPN}(X') \in \mathcal{T})$ holds in general.
\end{proof}

From \Cref{th:unlearn=retraining} we can directly conclude \Cref{cor:0cr}, which states a strong privacy guarantee of \textsc{UnlearnSPN} in the sense of $\epsilon$-certified removal.

\begin{corollary}\label{cor:0cr}
Let $X$ be a dataset, $x \in X$ a data point, $X' = X \setminus \{x\}$, and $\Phi = \textsc{LearnSPN}(X)$.
\textsc{UnlearnSPN}$(\Phi, x)$ is 0-CR for $\mathcal{A}=$ \textsc{LearnSPN}.
\end{corollary}

\begin{proof}[Proof of \Cref{cor:0cr}]
    \begin{align*}
        & ~ P(\textsc{UnlearnSPN}(\Phi, x) \in \mathcal{T}) = P(\textsc{LearnSPN}(X') \in \mathcal{T}) \\
        \Leftrightarrow & ~ \frac{P(\textsc{UnlearnSPN}(\Phi, x) \in \mathcal{T})}{P(\textsc{LearnSPN}(X \setminus \{x\}) \in \mathcal{T})} = 1 = e^{0} \\
        \Leftrightarrow & ~ e^0 \leq \frac{P(\textsc{UnlearnSPN}(\textsc{LearnSPN}(X), x) \in \mathcal{T})}{P(\textsc{LearnSPN}(X \setminus \{x\}) \in \mathcal{T})} \leq e^0
    \end{align*}
\end{proof}

\section{Experiments}\label{sec:experiments}
    In the following, we will empirically show that \textsc{UnlearnSPN} provides a better runtime than retraining an SPN from scratch and that the runtime of the here presented modified version of \textsc{LearnSPN} does not significantly differ from that of the original \textsc{LearnSPN}. All experiments were performed on a MacBook Pro with a 2.3 GHz Quad-Core Intel Core i7 processor and 32GB RAM. The implementation of our modified training algorithm, as well as \textsc{UnlearnSPN}, are based on SPFlow \cite{Molina2019SPFlow} and are publicly available on GitHub\footnote{\url{https://github.com/ROYALBEFF/UnlearnSPN}}.

For the former, we train an initial SPN on a random subset of 1000 data points of the training data and consecutively remove 100 data points at random afterward. Then, we compare the accumulated runtime of the 100 removal operations with that of retraining the model 100 times on the corresponding remaining data. We repeat this experiment 10 times and consider the mean as well as the standard deviation of the runtimes for comparison. The datasets used were Abalone, Adult, MSNBC, Plants, and Wine from the UCI Machine Learning Repository \cite{Dua:2019}. A short summary of the variables in each dataset is given in \Cref{tab:data}. All numerical variables were considered Gaussian.

\begin{table}
    \centering
    \caption{Number and types of variables in each dataset.}
    \begin{tabular}{|c|c|c|c|}\hline
         Dataset & \# Categorical & \# Gaussian & \# Total \\ \hline
         Abalone & 1 & 8 & 9 \\
         Adult & 9 & 6 & 15\\
         MSNBC & 17 & 0 & 17\\
         Plants & 71 & 0 & 71\\
         Wine & 1 & 13 & 14 \\ \hline
    \end{tabular}
    \label{tab:data}
\end{table}

Variables were split by computing the randomized dependency coefficient (RDC) \cite{lopez2013randomized} for each pair of variables, which gives us a symmetric coefficient matrix. Using a dependency threshold, we obtained an adjacency matrix. The connected components induced by this matrix form the variable split. Since RDC does not allow certified removal, we store the adjacency matrix and the random projections for each \textsc{SplitVariables} node, recompute the adjacency matrix and check if any of the connected components changed. If this is not the case, \textsc{UnlearnSPN} updates the state of the corresponding node. Otherwise, the sub-SPN must be retrained.

The data splits were computed using Q-$k$-Means \cite{ginart2019making} with $k=2$, which is a modified version of $k$-Means that allows certified data removal. The main idea behind Q-$k$-Means is that quantizing the cluster centroids reduces the impact each data point can have. Therefore, multiple data points can be removed without affecting the resulting centroids.

\begin{table}
    \centering
    \caption{Runtime comparison of removing 100 random data points via retraining and \textsc{UnlearnSPN}.}
    \begin{tabular}{|c|c|c|}
        \hline
        Dataset     & Retraining [s]    & Unlearning [s] \\ \hline
        Abalone     & 152.450 ± 10.217       & 111.540 ± 5.725 \\
        Adult       & 529.350 ± 64.865      & 466.950 ± 45.061 \\
        MSNBC       & 71.898 ± 19.101       & 55.636 ± 8.381\\
        Plants      & 6880.890 ± 605.628    & 6159.860 ± 609.329 \\
        Wine        & 0.249 ± 0.012         & 0.105 ± 0.003 \\ \hline
    \end{tabular}
    \label{tab:results}
\end{table}

\Cref{tab:results} contains the runtimes of unlearning and retraining for 100 random data points. We observe that in all of our experiments, removing data via unlearning is faster than retraining the model on the remaining data. The time improvements vary between approximately 58\% for the Wine dataset and 10\% for the Plants dataset. We also see that the time needed for unlearning varies less than for retraining.

\begin{table}
    \centering
    \caption{Comparing the training duration of the original \textsc{LearnSPN} and the modified, removal enabled \textsc{LearnSPN}.}
    \begin{tabular}{|c|c|c|}
        \hline
        Dataset     & Original \textsc{LearnSPN} [s]    & Modified \textsc{LearnSPN} [s] \\ \hline
        Abalone     & 1.463 ± 0.138     & 2.378 ± 0.247 \\
        Adult       & 4.7838 ± 0.783    & 6.238 ± 0.881 \\
        MSNBC       & 0.995 ± 0.171     & 1.578 ± 0.547 \\
        Plants      & 68.107 ± 7.927    & 71.312 ± 10.349 \\
        Wine        & 0.311 ± 0.013     & 0.0054 ± 0.001 \\ \hline
    \end{tabular}
    \label{tab:training_comparison}
\end{table}

For the runtime comparison of the original \textsc{LearnSPN} and our removal enabled version, we use the same datasets as above and perform both training algorithms 10 times on all data sets. The original \textsc{LearnSPN} algorithm uses RDC for splitting variables and $k$-Means with $k=2$ for splitting data points. Note that the SPFlow implementation of \textsc{LearnSPN} performs pruning by default. We explicitly omit the pruning step for a fairer comparison.
\Cref{tab:training_comparison} contains the runtimes of training an SPN with the original \textsc{LearnSPN} implementation in SPflow, and our modified, removal enabled version.
In \Cref{tab:training_comparison} we state our results of the runtime comparison, which show that the time needed to train an SPN using the removal enabled version of \textsc{LearnSPN} is only slightly longer than with the original \textsc{LearnSPN} algorithm. Other than the retraining and unlearning, the initial training is usually only performed once. Thus, the increase of the initial training duration amortizes after only a few data removals.
    
\section{Conclusion and Future Work}\label{sec:conclusion}
    We presented the first unlearning algorithm for SPNs, namely \textsc{UnlearnSPN}, which allows deleting the influence of specific data points perfectly on demand. With this, SPNs match the necessary requirements to comply with legal regulations such as the \textit{right to be forgotten} (GDPR). For this, we introduced a modified version of \textsc{LearnSPN}, which yields removal enabled SPNs. The major differences to the standard \textsc{LearnSPN} algorithm are the additional state that is stored for each node in the SPN and the fact that the SPN must not be pruned afterward. Our experiments showed that our modified version of \textsc{LearnSPN} slightly increases the training duration. However, the additional time needed in training is quickly amortized by the speed up \textsc{UnlearnSPN} provides over retraining the SPN from scratch when removing data points. Note that even though we only argued about removing single data points, this approach can easily be generalized to sets of data points and therefore allows batch removal. Furthermore, the approach can also be adapted so that data points can be removed and added to the SPN. This is especially interesting in online scenarios.

There are still some open questions that could be addressed in future work. Here, we assumed the leaf nodes to represent univariate distributions. However, leaf nodes are generally not restricted to univariate distribution but could also represent multivariate distribution or be replaced by Chow-Liu trees. \textsc{UnlearnSPN} could also be adapted to handle other kinds of leaf nodes and thus be applicable in even more scenarios.

Another point that could be addressed in future work is that \textsc{UnlearnSPN} only works on unpruned SPNs, which implies larger models and possibly slower inference. \textsc{UnlearnSPN} might be able to work on pruned SPNs as well if we find a way to combine states during pruning such that all information necessary for removal is preserved and applicable.

Finally, it would be interesting to compare unlearning to full differential private training since unlearning can be seen as a trade-off between retraining and differential privacy. However, to our best knowledge, there is no differential private training algorithm for SPNs yet.
    
\bibliography{main.bib}

\end{document}